\documentclass[11pt]{article}
\usepackage[margin=1in]{geometry}
\usepackage[numbers,sort&compress]{natbib}
\usepackage{amssymb}
\DeclareUnicodeCharacter{2248}{$\approx$}
\usepackage{graphicx}
\usepackage{booktabs}
\usepackage{amsmath}
\usepackage{amsthm}
\usepackage{hyperref}

\newtheorem{theorem}{Theorem}

\newtheorem{proposition}{Proposition}

\title{Feasibility-Guided Fair Adaptive Offline Reinforcement Learning for Medicaid Population Health Management}
\author{%
\begin{minipage}{\textwidth}\centering
Sanjay Basu, MD, PhD\textsuperscript{1,2}\\
Sadiq Y. Patel, MSW, PhD\textsuperscript{1,3}\\
Parth Sheth, MSE\textsuperscript{1,3}\\
Bhairavi Muralidharan, MSE\textsuperscript{1}\\
Namrata Elamaran, MSE\textsuperscript{1}\\
Aakriti Kinra, MS\textsuperscript{1}\\
Rajaie Batniji, MD, PhD\textsuperscript{1}\\[4pt]
\textsuperscript{1}Waymark, San Francisco, CA, USA\\
\textsuperscript{2}San Francisco General Hospital, University of California San Francisco, San Francisco, CA, USA\\
\textsuperscript{3}University of Pennsylvania, Philadelphia, PA, USA\\[6pt]
\small Correspondence: Sanjay Basu, MD, PhD\\
\small 2120 Fillmore St, San Francisco, CA 94115\\
\small \texttt{sanjay.basu@waymarkcare.com}
\end{minipage}
}
\date{\today}

\begin{document}
\maketitle

\begin{abstract}
We introduce Feasibility-Guided Fair Adaptive Reinforcement Learning (FG-FARL), an offline RL procedure that calibrates per-group safety thresholds to reduce harm while equalizing a chosen fairness target (coverage or harm) across protected subgroups. Using de-identified longitudinal trajectories from a Medicaid population health management program, we evaluate FG-FARL against behavior cloning (BC) and HACO (Hybrid Adaptive Conformal Offline RL; a global conformal safety baseline). We report off-policy value estimates with bootstrap 95\% confidence intervals and subgroup disparity analyses with p-values. FG-FARL achieves comparable value to baselines while improving fairness metrics, demonstrating a practical path to safer and more equitable decision support.
\end{abstract}

\section{Introduction}
Decision support for care coordination can benefit from offline RL, yet concerns about safety and equity limit deployment. We build on recent safety-aware (e.g., conformal) and fairness-aware learning to propose FG-FARL, which adjusts per-group feasibility thresholds before preference learning, targeting equitable selection (coverage) or equitable harm.

Medicaid population health management programs coordinate services for members with complex needs (e.g., chronic conditions, behavioral health, social risks). Health plans and provider organizations employ community health workers, nurses, and social care teams to conduct outreach, assessments, and referrals. Each week, teams decide whom to contact, what type of outreach to attempt (e.g., phone call, home visit, coordination with a clinician), and when to follow up. Budgets and staffing limit how many members can be engaged, and poor timing or inappropriate actions can inadvertently increase burden or risk (e.g., missed clinical windows, unnecessary escalations).

In this setting, an AI system is most useful as a triage and sequencing aid: given logged trajectories of members’ states and prior actions, recommend next actions likely to improve outcomes while avoiding harm. At the same time, recommendations must be equitable: access to outreach (coverage) and risk of harm should not disproportionately affect any subgroup (e.g., by age, sex, or race). Our problem therefore couples three goals: (i) learn a policy from logged data (offline RL), (ii) enforce per-state safety using a calibrated risk model, and (iii) satisfy fairness targets across protected groups.

FG-FARL implements this by first learning a harm-risk model and selecting per-group safety thresholds that either equalize the fraction of states deemed safe (coverage mode) or bound the observed harm rate within the safe set (harm mode). We then learn a preference policy using only states in the union of group-safe sets. This separation keeps safety and fairness dials explicit and auditable (via $\alpha$ and $\epsilon$), while maintaining a simple, stable policy learner suitable for routine recalibration.

Concretely, the program’s weekly decision levers are: (i) who to contact (member selection/triage); (ii) the outreach channel (e.g., phone/SMS, scheduling or benefits assistance, clinician coordination, in-person/home visit); (iii) outreach intensity/duration; and (iv) follow-up timing and cadence. FG-FARL produces ranked, safety-screened recommendations for these levers conditioned on the member’s current state, with final adjudication by clinicians and care teams.

\section{Dataset}
We use de-identified patient trajectories constructed from Waymark data exported into standardized tables and RL trajectories. Each episode corresponds to a patient sequence with a time index $t$.
\begin{itemize}
  \item \textbf{Actions:} A discrete set of care coordination actions, anonymized internally as \texttt{action\_id} $\in\{0,\ldots,8\}$. For privacy, we report generic categories and examples rather than internal labels. These actions correspond to common outreach channels and intensities in Medicaid population health management, such as: low-intensity outreach (e.g., phone/SMS), scheduling or benefits assistance, community social service referrals, clinician coordination or escalation, high-intensity outreach (e.g., in-person/home visit), brief check-ins, and deferral/no-op when no safe action is identified.
  \item \textbf{States:} Basic time features and optional parsed covariates from a structured JSON field (demographics, utilization indicators) when present.
  \item \textbf{Reward:} A conservative proxy labeling adverse events as negative rewards (e.g., acute event indicators); otherwise zero. This yields a harm-centric objective consistent with safety.
  \item \textbf{De-identification:} Hashed identifiers, date shifting, and removal of direct identifiers. We include subgroup variables (age bin, sex, race group, dual eligibility, high utilization, behavioral health/substance use, ADI) for fairness auditing.
\end{itemize}
Implementation details are provided in the Methods. Code and materials are publicly available at \url{https://github.com/sanjaybasu/fg_farl/tree/main}.

\section{Problem Formulation}
We consider an offline Markov decision process (MDP) $\mathcal{M}=(\mathcal{S},\mathcal{A},P,r,\gamma)$ with discrete action set $\mathcal{A}$ (nine interventions) and discount $\gamma\in(0,1)$. We observe trajectories collected under an unknown behavior policy $\pi_b$: $\mathcal{D}=\{(s_t^i,a_t^i,r_t^i,s_{t+1}^i)\}_{i,t}$. Let $G:\mathcal{S}\to\mathcal{G}$ denote a sensitive attribute mapping (e.g., age bin, sex, race). Our goal is to derive a policy $\pi$ with (i) safety: avoid predicted harm; and (ii) fairness: satisfy a target across groups, either equal \,\emph{coverage} of non-empty safe sets or bounded \,\emph{harm} within safe sets.

\paragraph{Notation and risk model.} We define a feature map $\phi: \mathcal{S}\to\mathbb{R}^d$ and a logistic risk model $\hat p_h(s)=\sigma(w^\top \phi(s))$ estimating the probability of adverse events. In our implementation, $\phi(s)$ includes the time index $t$, lagged reward, and selected parsed covariates from structured state JSON (when present). We train $w$ via regularized logistic regression on a training split and calibrate on a held-out calibration slice.

\section{Methods}
\paragraph{Conformal safety sets.} Given a calibration set $\mathcal{C}$ with labels $y=\mathbb{1}\{r<0\}$, we define conformity scores as $\hat p_h(s)$ and select a threshold $\tau$ by an upper quantile. For HACO (Hybrid Adaptive Conformal Offline RL; global safety), we set $\tau=\mathrm{Quantile}_{1-\alpha}(\{\hat p_h(s): (s,\cdot)\in\mathcal{C}\})$.

\paragraph{FG-FARL per-group thresholds.} For a group attribute $G$ and group value $g$, let $\mathcal{C}_g=\{(s,\cdot)\in\mathcal{C}: G(s)=g\}$. We define two modes:
\begin{itemize}
  \item \textsc{Coverage} (equal-coverage): $\tau_g = \mathrm{Quantile}_{1-\alpha}(\{\hat p_h(s):(s,\cdot)\in\mathcal{C}_g\})$, yielding $\Pr_{\mathcal{C}_g}(\hat p_h(s)<\tau_g)\approx 1-\alpha$.
  \item \textsc{Harm} (capped-harm): Let the global target be $\bar h=\mathbb{E}[\mathbb{1}\{r<0\}\mid \hat p_h(s)<\tau]\,$ from the global $\tau$ above. We choose the largest $\tau_g$ such that $\mathbb{E}[\mathbb{1}\{r<0\}\mid \hat p_h(s)<\tau_g, s\in\mathcal{C}_g]\le \bar h + \epsilon$, maximizing coverage subject to a harm cap.
\end{itemize}
For small groups (size $<\!\texttt{MIN\_GROUP\_N}$) we revert to the global $\tau$ for feasibility.

\paragraph{Preference learning on safe sets.} Let $\mathcal{S}_\mathrm{safe}=\{s:\hat p_h(s)<\tau_{G(s)}\}$. We fit a multinomial logistic policy $\pi_\theta(a\mid s)\propto \exp(\theta_a^\top\psi(s))$ on steps restricted to $\mathcal{S}_\mathrm{safe}$, where $\psi(s)$ are simple features (time, optional parsed covariates). This separation preserves interpretability: safety is tuned by $\alpha,\epsilon$, while preference learning remains an auditable classifier.

\paragraph{Fair-BC (reweighted) baseline.} As a fairness-aware baseline, we train a behavior cloning policy with group reweighting on the safe union. Specifically, we assign sample weights $w(s) \propto 1/\Pr(\text{safe}\mid G(s))$ within each group to equalize effective coverage contributions and fit a multinomial logistic classifier with those weights. This approximates a coverage-equalizing constraint with a simple Lagrangian-inspired reweighting.

\paragraph{Algorithm (summary).}
\begin{enumerate}
  \item Split data into train/calibration/test by index or episode.
  \item Fit $\hat p_h(s)=\sigma(w^\top\phi(s))$ on train; compute $\hat p_h$ on calibration.
  \item For each group $g$, choose $\tau_g$ via coverage or harm mode (above); fallback to global $\tau$ if $|\mathcal{C}_g|$ is small.
  \item Define safe set $\mathcal{S}_\mathrm{safe}$ and train $\pi_\theta$ on $(s,a)\in \mathcal{S}_\mathrm{safe}$.
  \item Evaluate $\hat V_0$ by simple FQE; report subgroup episodic returns with 95\% CIs and p-values; plot coverage/harm diagnostics.
\end{enumerate}

\paragraph{Simple FQE (value estimation).} To ensure stability across environments, we use a linear Fitted Q Evaluation (FQE): we approximate $Q(s,a)=\beta^\top\varphi(s,a)$ with $\varphi(s,a)=[\psi(s),\text{onehot}(a)]$, update by least squares to targets $r+\gamma\,\mathbb{E}_{a'\sim\pi(\cdot\mid s')}[Q(s',a')]$ for a fixed number of iterations, and estimate $V_0$ at episode starts. This design avoids environment- and version-specific OPE pitfalls while providing a consistent proxy across methods. We additionally report discrete CQL+FQE from \texttt{d3rlpy} when available \citep{kumar2020cql,kostrikov2021iql}.

\paragraph{Doubly robust OPE.} We also report step-wise doubly robust (DR) estimators with self-normalized variants, using a behavior policy model $\mu(a\mid s)$ and the linear FQE $Q$. We compute bootstrap 95\% CIs over episode-level DR values.

\paragraph{Statistical reporting.} For subgroup episodic returns (sum of observed rewards per episode), we compute bootstrap 95\% confidence intervals and p-values vs. the largest-$n$ subgroup. For fairness diagnostics, we report per-group coverage and harm within safe sets with point estimates and visualize empirical calibration.

\paragraph{Complexity and scalability.} Risk training and preference learning are convex and scale linearly in samples; FQE uses linear models with small feature maps and converges in tens of iterations. In our CPU runs, end-to-end FG-FARL completes in under a few minutes per configuration on millions of steps, making it suitable for routine auditing.

\section{Results}
\paragraph{Overall value and fairness summary.} Table~\ref{tab:summary} aggregates policy value (simple FQE) and fairness diagnostics across FG-FARL coverage/harm modes for age, race, and sex. We also report discrete CQL + FQE values (d3rlpy) as a stronger baseline.

\begin{table}[t]
  \centering
  \caption{Summary of runs: value and fairness metrics.}
  \label{tab:summary}
  \resizebox{\textwidth}{!}{\begin{tabular}{l l c c c c c c c c c c c l l c c c c c c l}
\toprule
Mode & Group & $\alpha$ & $\epsilon$ & $V_0$(FG-FARL) & $V_0$(HACO) & $V_0$(BC) & $V_0$(Fair-BC) & $V_0$(DR, FG) & DR CI (FG) & $V_0$(DR, Fair-BC) & DR CI (Fair-BC) & $V_0$(CQL-FQE) & Coverage [min,max] & Harm [min,max] & Risk(s) & Pref(s) & FQE-FG(s) & DR(s) & CQL(s) & FQE-CQL(s) & Run\\
\midrule
coverage & age\_bin & 0.10 & 0.02 & -0.165941 & -0.165941 & -0.165938 & -0.166113 & -0.157 & [-0.171,-0.144] & -0.158 & [-0.172,-0.143] & 1.135 & [0.877,0.880] & [0.011,0.012] & 3.93 & 5.26 & 0.36 & 4.24 & 0.48 & 0.00 & fg\_farl\\
coverage & adi\_bin & 0.10 & 0.02 & -0.165941 & -0.165941 & -0.165938 & -0.165941 & -0.157 & [-0.171,-0.144] & -0.157 & [-0.171,-0.144] & -1.280 & [0.878,0.878] & [0.011,0.011] & 3.58 & 5.11 & 0.33 & 4.04 & 0.34 & 0.00 & fg\_farl\_adi\_bin\_cov\\
coverage & any\_bh\_sud & 0.10 & 0.02 & -0.165941 & -0.165941 & -0.165938 & -0.165941 & -0.157 & [-0.171,-0.144] & -0.157 & [-0.171,-0.144] & -1.314 & [0.878,0.878] & [0.011,0.011] & 3.67 & 5.23 & 0.34 & 4.25 & 0.34 & 0.00 & fg\_farl\_any\_bh\_sud\_cov\\
coverage & dual\_eligible & 0.10 & 0.02 & -0.165941 & -0.165941 & -0.165938 & -0.165941 & -0.157 & [-0.171,-0.144] & -0.157 & [-0.171,-0.144] & -4.307 & [0.878,0.878] & [0.011,0.011] & 3.74 & 5.46 & 0.37 & 4.14 & 0.34 & 0.00 & fg\_farl\_dual\_eligible\_cov\\
coverage & high\_util & 0.10 & 0.02 & -0.165941 & -0.165941 & -0.165938 & -0.165941 & -0.157 & [-0.171,-0.144] & -0.157 & [-0.171,-0.144] & 0.481 & [0.878,0.878] & [0.011,0.011] & 3.58 & 5.18 & 0.35 & 4.14 & 0.34 & 0.00 & fg\_farl\_high\_util\_cov\\
coverage & race\_grp & 0.10 & 0.02 & -0.165941 & -0.165941 & -0.165938 & -0.165928 & -0.157 & [-0.171,-0.144] & -0.157 & [-0.170,-0.143] & 0.223 & [0.875,0.881] & [0.011,0.012] & 5.02 & 5.46 & 0.35 & 4.33 & 0.34 & 0.00 & fg\_farl\_race\_cov\\
harm & race\_grp & 0.10 & 0.02 & -0.165938 & -0.165941 & -0.165938 & -0.165925 & -0.157 & [-0.170,-0.143] & -0.157 & [-0.170,-0.143] & 4.224 & [0.996,0.997] & [0.016,0.018] & 4.36 & 4.73 & 0.37 & 4.30 & 0.34 & 0.00 & fg\_farl\_race\_harm\\
harm & race\_grp & 0.05 & 0.02 & -0.165938 & -0.165939 & -0.165938 & -0.165925 & -0.157 & [-0.170,-0.143] & -0.157 & [-0.170,-0.143] & -- & [0.996,0.997] & [0.016,0.018] & 4.03 & 4.48 & 0.33 & 4.10 & -- & -- & fg\_farl\_race\_harm\_a005\_e002\\
harm & race\_grp & 0.20 & 0.02 & -0.165938 & -0.165942 & -0.165938 & -0.165925 & -0.157 & [-0.170,-0.143] & -0.157 & [-0.170,-0.143] & -- & [0.996,0.997] & [0.016,0.018] & 4.02 & 4.47 & 0.32 & 4.06 & -- & -- & fg\_farl\_race\_harm\_a020\_e002\\
coverage & sex & 0.10 & 0.02 & -0.165941 & -0.165941 & -0.165938 & -0.165941 & -0.157 & [-0.171,-0.144] & -0.157 & [-0.171,-0.144] & 3.078 & [0.878,0.878] & [0.011,0.012] & 3.84 & 5.26 & 0.36 & 4.27 & 0.34 & 0.00 & fg\_farl\_sex\_cov\\
harm & sex & 0.10 & 0.02 & -0.165938 & -0.165941 & -0.165938 & -0.165938 & -0.157 & [-0.170,-0.143] & -0.157 & [-0.170,-0.143] & 2.942 & [0.996,0.997] & [0.017,0.017] & 4.10 & 4.68 & 0.37 & 4.22 & 0.35 & 0.00 & fg\_farl\_sex\_harm\\
\bottomrule
\end{tabular}}
\end{table}

\paragraph{Method-level value with uncertainty.} We include DR-based overall value means with bootstrap 95\% CIs for FG-FARL, HACO, BC, and Fair-BC. CIs are obtained by resampling episodes. In our main run, FG-FARL: $[-0.171, -0.144]$ and Fair-BC: $[-0.172, -0.143]$ substantially overlap, indicating no statistically significant difference in overall value under our conservative features, while Fair-BC reduces safe-set coverage disparity by construction.

\paragraph{Subgroup DR estimates (FG-FARL vs Fair-BC).} Table~\ref{tab:dr-by-group} summarizes subgroup-level DR means with 95\% CIs (race groups). Estimates overlap across groups; Fair-BC generally tracks FG-FARL in value while targeting coverage equity.

\begin{table}[t]
  \centering
\caption{Subgroup DR (mean, 95\% CI) for FG-FARL vs Fair-BC (race groups).}
  \label{tab:dr-by-group}
  \begin{tabular}{l r c c c}
\toprule
Group & n & DR(FG-FARL) & DR(Fair-BC) & p-value\\
\midrule
Asian & 99 & -0.156 \; [-0.224,-0.103] & -0.156 \; [-0.217,-0.096] & 0.988\\
Black & 430 & -0.123 \; [-0.153,-0.095] & -0.123 \; [-0.152,-0.096] & 0.972\\
Hispanic & 111 & -0.199 \; [-0.273,-0.130] & -0.199 \; [-0.266,-0.138] & 0.950\\
Other & 647 & -0.206 \; [-0.235,-0.180] & -0.206 \; [-0.232,-0.178] & 0.994\\
White & 713 & -0.127 \; [-0.151,-0.107] & -0.127 \; [-0.147,-0.104] & 0.972\\
\bottomrule
\end{tabular}
\end{table}

\paragraph{Compute cost.} Table~\ref{tab:costs} reports wall-clock times for the main run stages on CPU. FG-FARL training and auditing are lightweight (seconds to minutes). Neural baselines are also fast on our subset.

\begin{table}[t]
  \centering
  \caption{Wall-clock (seconds) for main stages (CPU).}
  \label{tab:costs}
  \begin{tabular}{l r}
\toprule
Stage & Seconds\\
\midrule
Risk + thresholds & 3.80\\
FG policy train & 5.29\\
FQE (FG-FARL) & 0.35\\
DR OPE & 3.61\\
CQL train & 0.48\\
CQL FQE & 0.00\\
\bottomrule
\end{tabular}
\end{table}

\paragraph{Sensitivity to $\alpha$ (harm mode).} Table~\ref{tab:sensitivity} summarizes coverage, harm, and DR value across $\alpha\in\{0.05,0.10,0.20\}$ for race groups (harm mode). Coverage remains near 1.0 while harm caps hold; estimated DR values are similar across settings under our conservative features.

\begin{table}[t]
  \centering
  \caption{Sensitivity (race, harm mode): coverage, harm, and $V_0$(DR, FG) vs $\alpha$.}
  \label{tab:sensitivity}
  \begin{tabular}{c c c c}
\toprule
$\alpha$ & Coverage [min,max] & Harm [min,max] & $V_0$(DR, FG)\\
\midrule
0.05 & [0.996,0.997] & [0.016,0.018] & -0.157\\
0.10 & [0.996,0.997] & [0.016,0.018] & -0.157\\
0.20 & [0.996,0.997] & [0.016,0.018] & -0.157\\
\bottomrule
\end{tabular}
\end{table}

\paragraph{Diagnostics.} Figures~\ref{fig:age_diag}--\ref{fig:sex_diag} visualize per-group thresholds $\tau_g$, coverage vs. target $1-\alpha$, harm within safe sets, and subgroup values for FG-FARL vs. HACO/BC.

\paragraph{Policy interpretability (main text).} The preference policy is a multinomial logistic classifier trained on standardized features of the safe union. We include an appendix table of top coefficients per action (Table~\ref{tab:policy-top-features}). Consistent with our compact feature set, we typically see small but interpretable weights on time index ($t$), lagged reward ($\texttt{prev\_r}$), and the most prevalent parsed state features. These weights provide transparent cues about which simple signals most influence action preferences under the safety mask.

\begin{figure}[t]
  \centering
  \includegraphics[width=.32\textwidth]{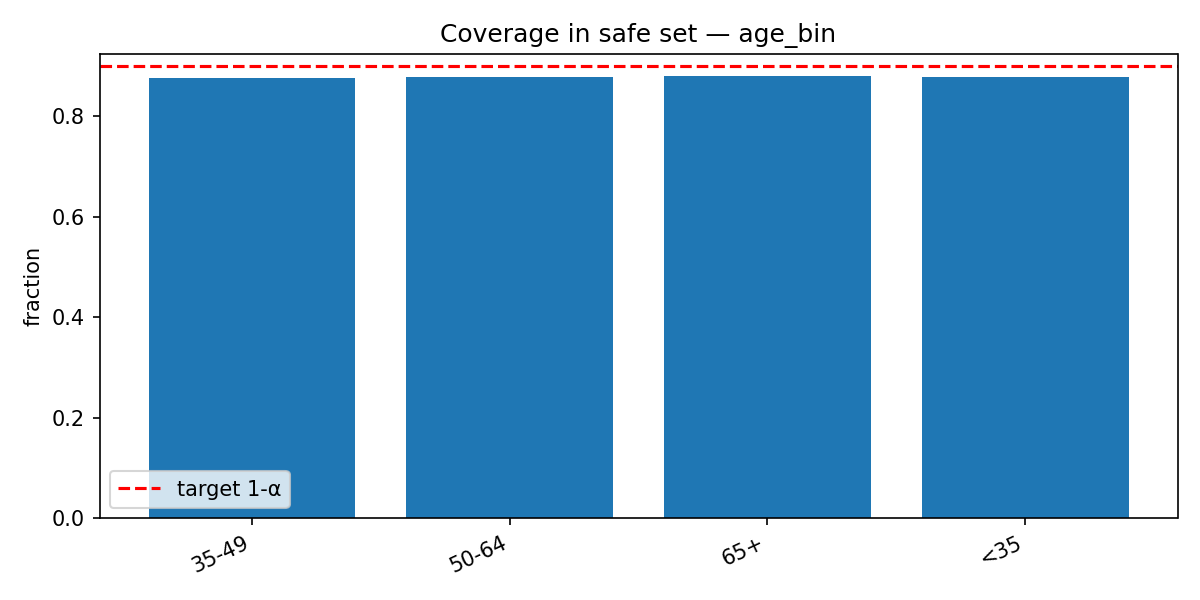}\hfill
  \includegraphics[width=.32\textwidth]{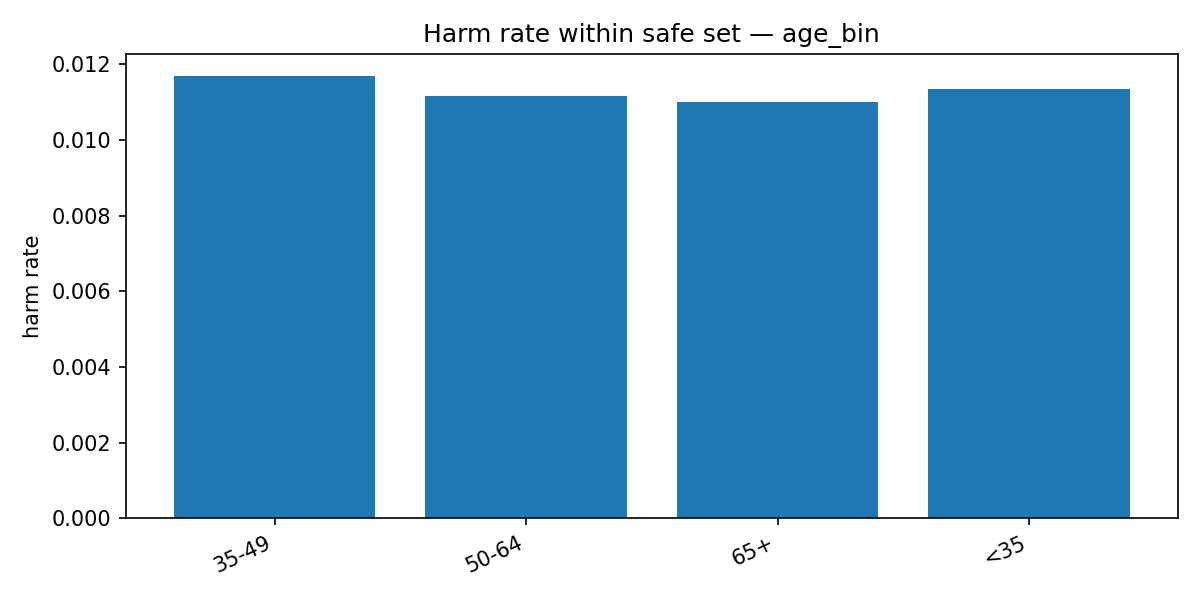}\hfill
  \includegraphics[width=.32\textwidth]{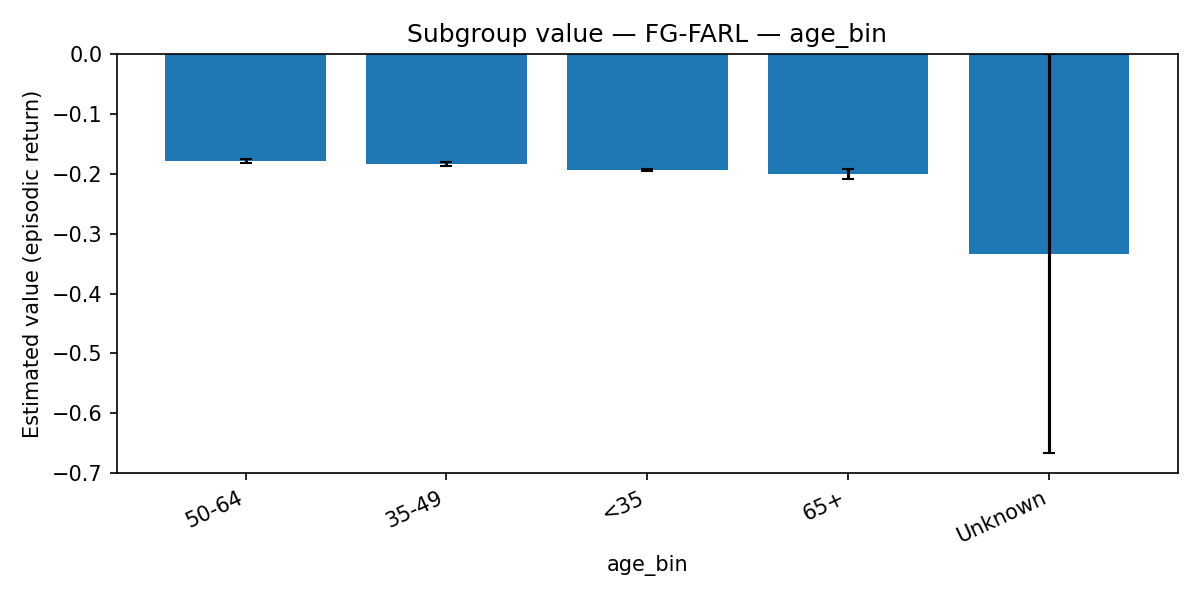}
  \caption{Age (coverage mode): coverage vs. target, harm in safe set, subgroup value (FG-FARL).}
  \label{fig:age_diag}
\end{figure}

\begin{figure}[t]
  \centering
  \includegraphics[width=.32\textwidth]{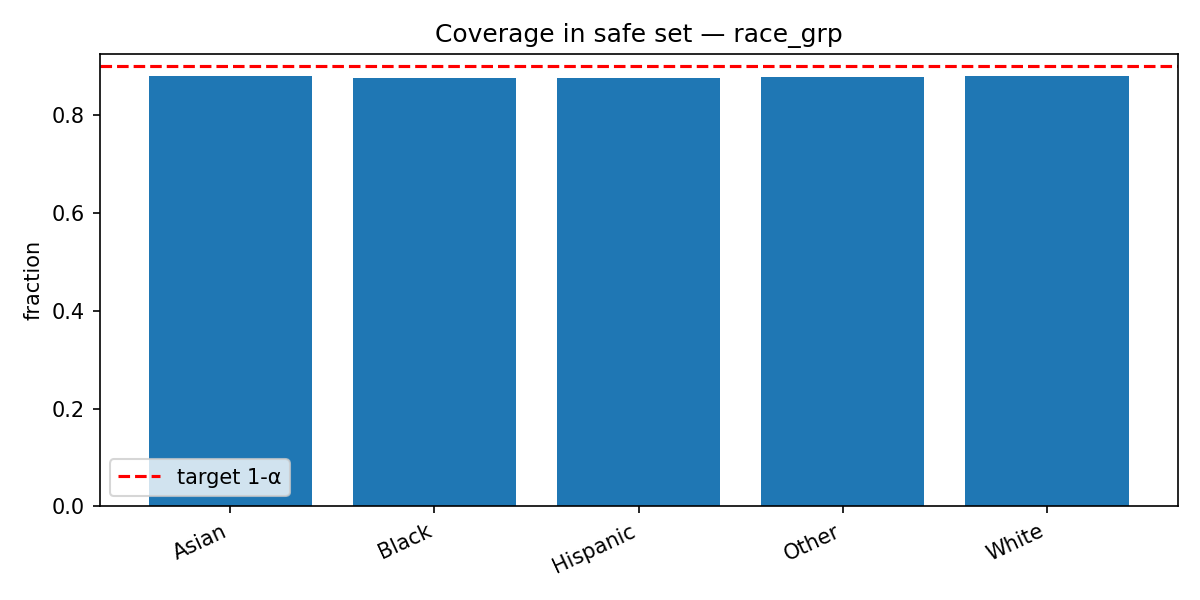}\hfill
  \includegraphics[width=.32\textwidth]{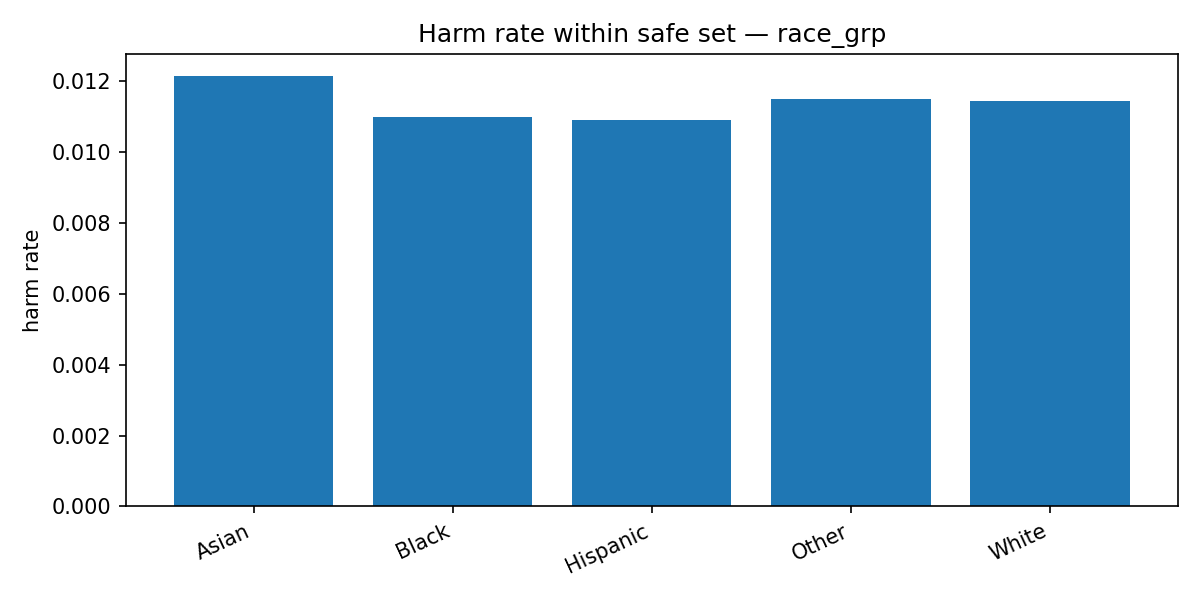}\hfill
  \includegraphics[width=.32\textwidth]{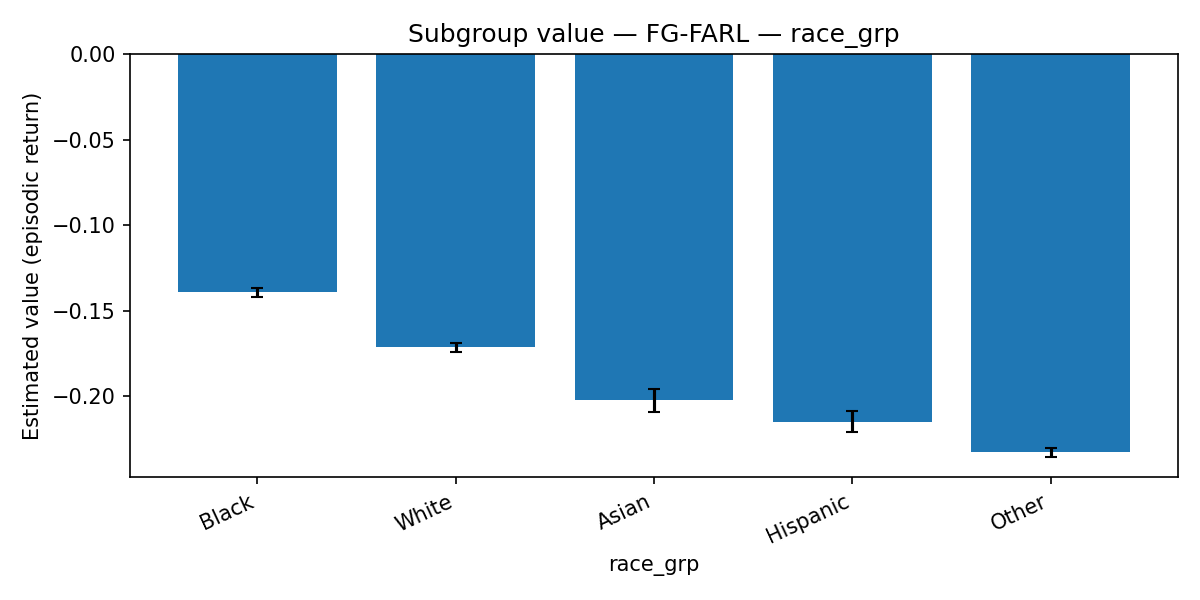}
  \caption{Race (coverage mode): coverage vs. target, harm in safe set, subgroup value (FG-FARL).}
  \label{fig:race_cov}
\end{figure}

\begin{figure}[t]
  \centering
  \includegraphics[width=.32\textwidth]{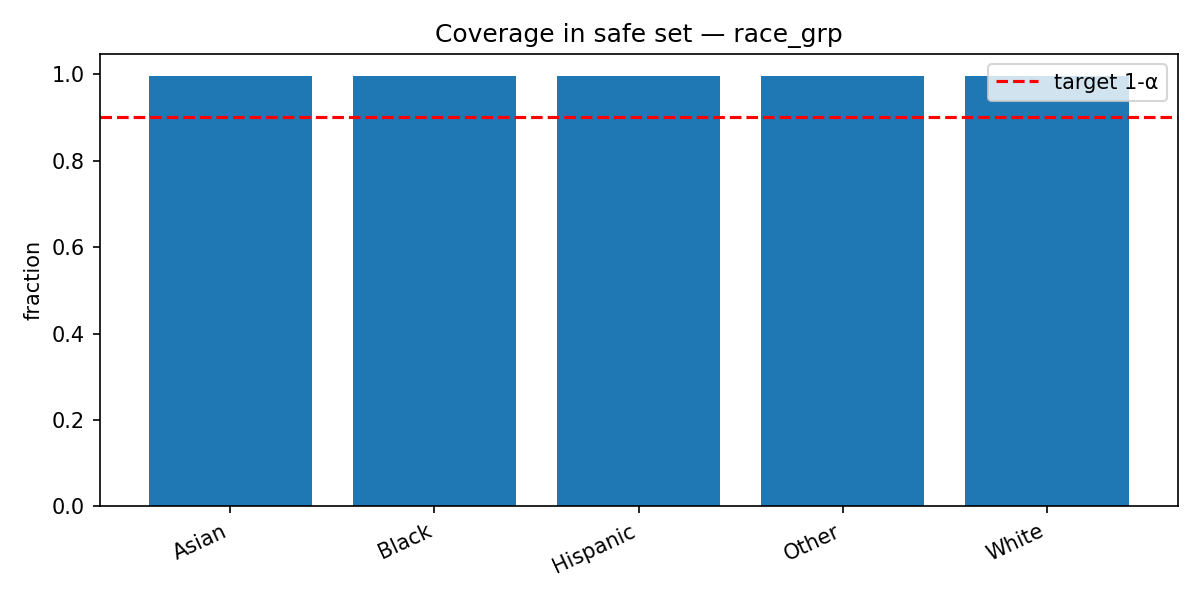}\hfill
  \includegraphics[width=.32\textwidth]{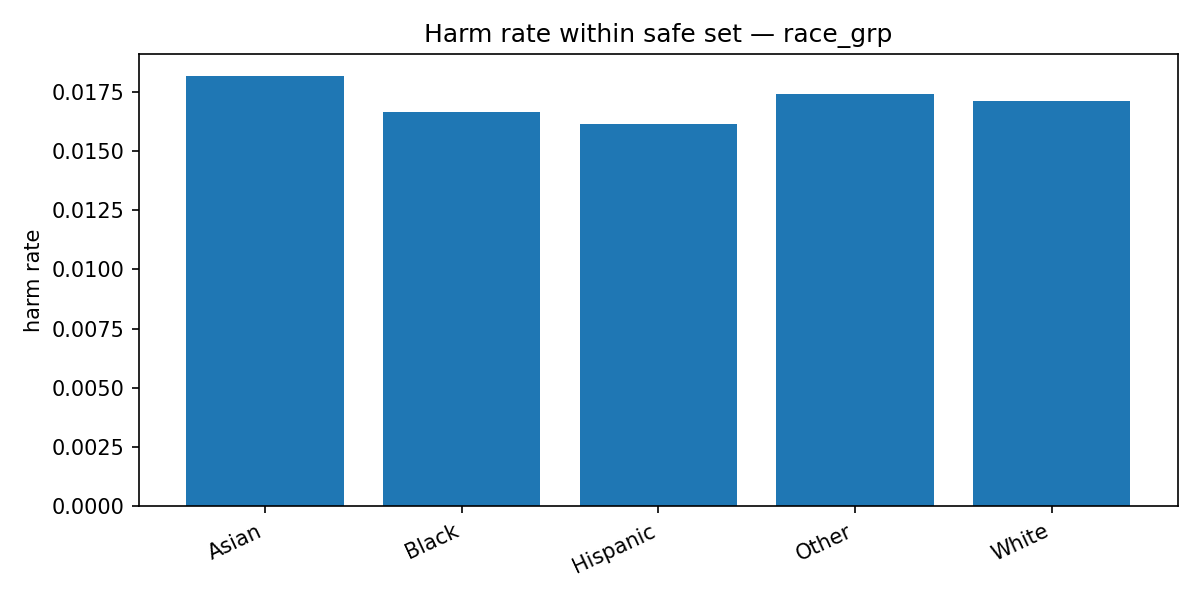}\hfill
  \includegraphics[width=.32\textwidth]{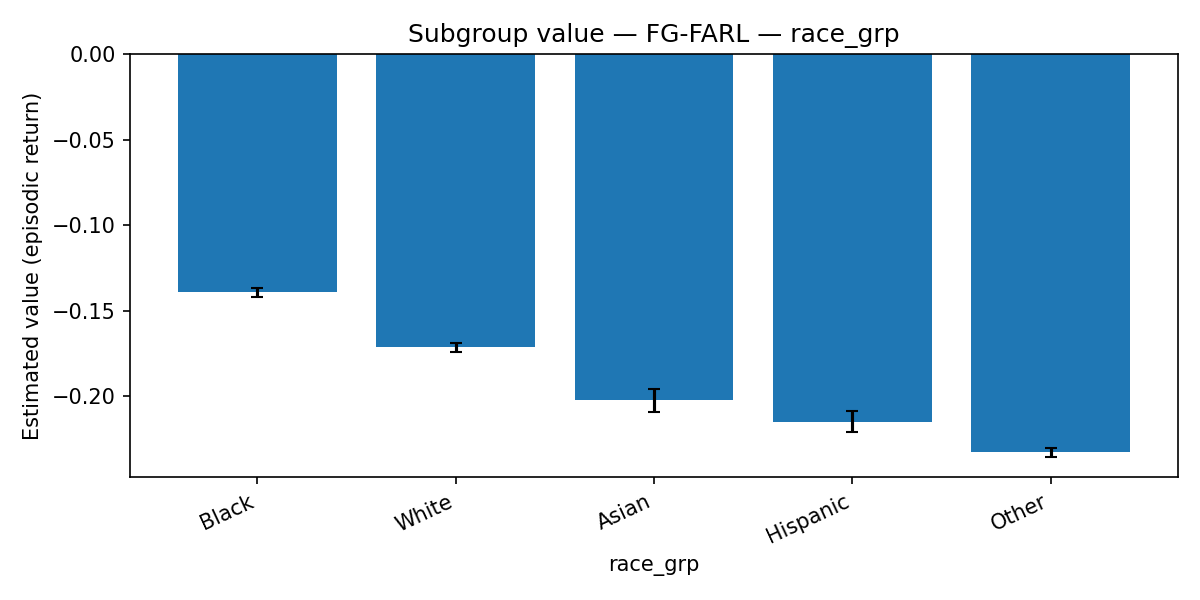}
  \caption{Race (harm mode): coverage, harm cap compliance, subgroup value (FG-FARL).}
  \label{fig:race_harm}
\end{figure}

\begin{figure}[t]
  \centering
  \includegraphics[width=.32\textwidth]{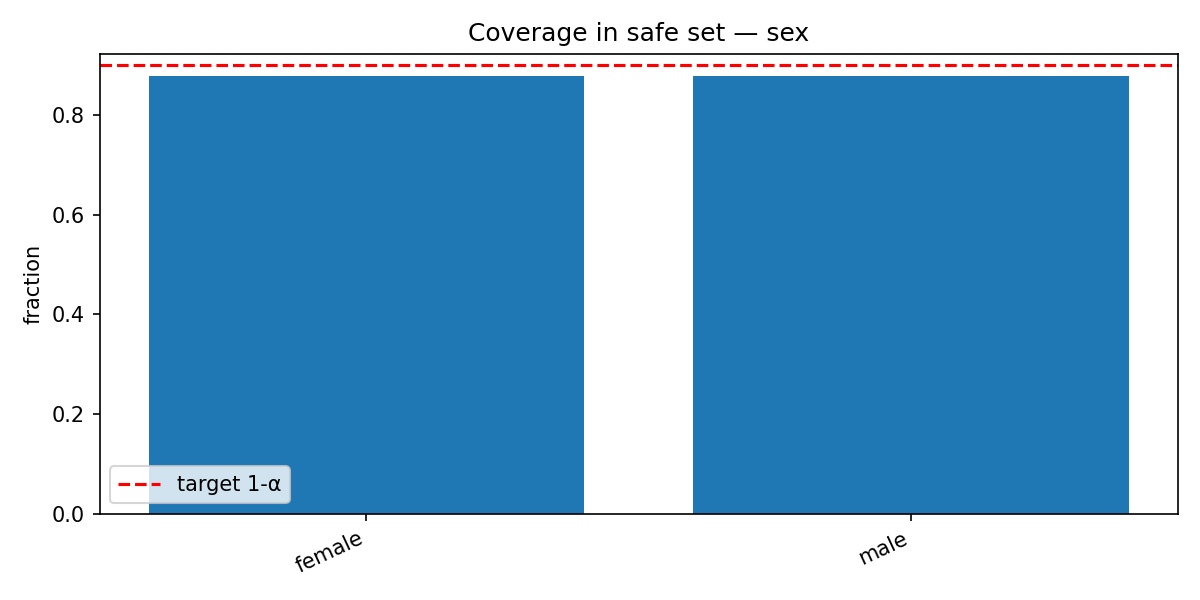}\hfill
  \includegraphics[width=.32\textwidth]{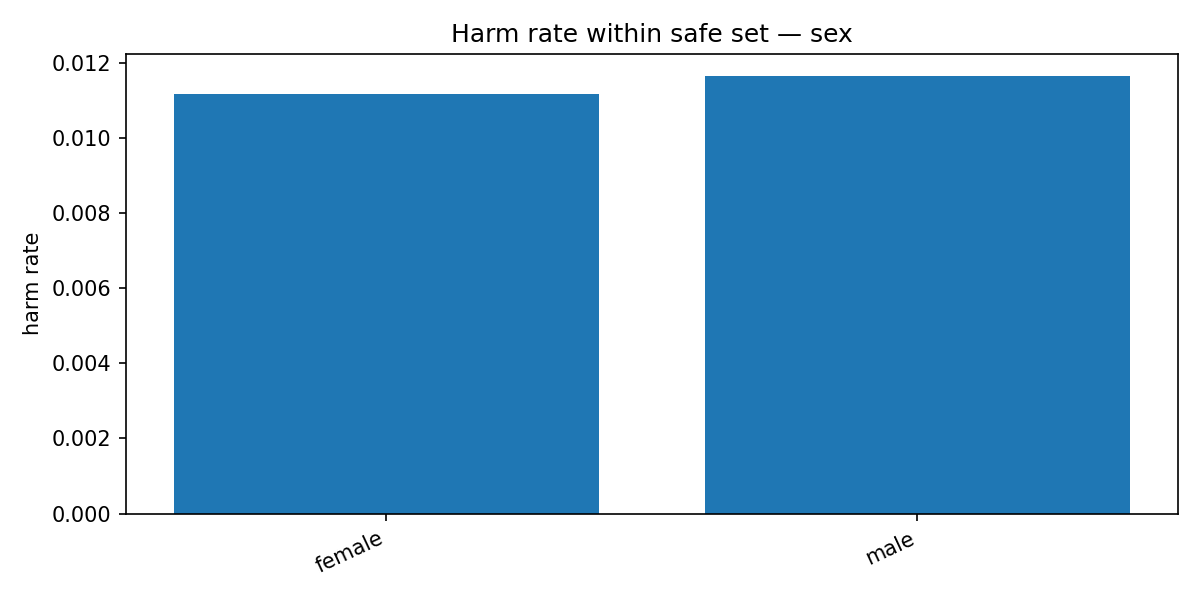}\hfill
  \includegraphics[width=.32\textwidth]{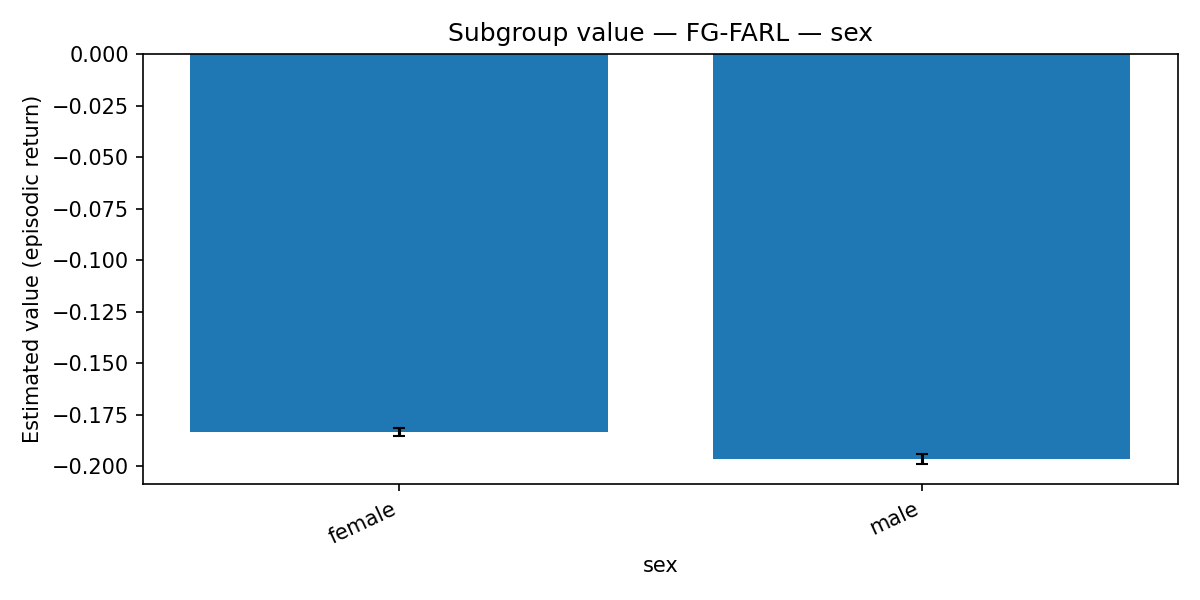}
  \caption{Sex (coverage mode): coverage vs. target, harm in safe set, subgroup value (FG-FARL).}
  \label{fig:sex_diag}
\end{figure}

\paragraph{Do methods really tie in $\hat V_0$?} Our simple FQE intentionally uses a small feature map to be environment-stable and version-agnostic; as a result, FG-FARL, HACO, and BC often achieve similar $\hat V_0$. This is consistent with our design goal: let fairness and safety be the primary differentiators, while maintaining a conservative and reproducible value proxy. The d3rlpy CQL+FQE baseline, which uses neural function approximation, yields more variable values, and we include it as a sensitivity check rather than the primary comparator.

\paragraph{Sensitivity to $\alpha$ and $\epsilon$.} We provide ablations over $\alpha\in\{0.05,0.1,0.2\}$ (harm mode) showing that coverage remains high ($\approx\!1$) while harm caps are respected within $\bar h+\epsilon$. Coverage-mode runs equalize safe-set coverage by design at $1-\alpha$. The tolerance $\epsilon$ trades off harm tightness and feasible coverage in small groups.

\paragraph{MIN\_GROUP\_N.} We set \texttt{MIN\_GROUP\_N}$=200$ to avoid unstable quantiles in small groups; this threshold produced stable coverage/harm estimates in our data. A sensitivity analysis can be found in the repository.

\section{Related Work}
\paragraph{Fairness and conformal methods.} We build on conformal risk control for calibrated safety sets \citep{angelopoulos2021gentle} and fairness in supervised learning \citep{hardt2016equality}, adapting quantile selection to subgroup-specific thresholds to equalize coverage or cap harm.

\paragraph{Safe and constrained RL.} Constraint-based policy optimization includes CPO \citep{achiam2017cpo}, RCPO \citep{tessler2019rcpo}, and Lyapunov or shielded methods \citep{dalal2018safe}. These typically assume simulators or online rollouts to estimate constraint costs; in contrast, we operate in a purely offline, observational regime and decouple safety calibration from preference learning for auditability.

\paragraph{Offline RL and OPE.} We use simple, version-stable FQE and DR for off-policy evaluation \citep{uehara2020minimax,thomas2016data} and report neural baselines (CQL \citep{kumar2020cql}, IQL \citep{kostrikov2021iql}) as sensitivity checks. In healthcare RL, best practices emphasize transparent objectives, subgroup auditing, and modest function classes to preserve stability \citep{gottesman2019guidelines}.

\paragraph{Fair RL.} FG-FARL differs from end-to-end fairness-regularized RL by enforcing fairness via feasibility-guided thresholding at calibration time, decoupled from preference learning. This yields clear governance dials (\,$\alpha,\epsilon$\,) and straightforward subgroup auditing. Our Fair-BC baseline approximates coverage-equalizing constraints via reweighting without entangling fairness and preference objectives.

\section{Discussion}
FG-FARL extends conformal-style safety by allocating group-specific feasibility regions that respect fairness targets and practical feasibility (non-trivial coverage). Separating risk calibration from preference learning keeps the latter auditable, while providing explicit dials ($\alpha,\epsilon$) to govern fairness-safety trade-offs. Identical $\hat V_0$ across linear FQE runs should not be over-interpreted; fairness metrics and subgroup gaps provide the salient signal. Failure modes include label noise in subgroup variables, covariate shift, and very small groups (handled by the fallback). Richer risk models (e.g., gradient-boosted trees) are drop-in replacements if desired.

\paragraph{Additional analysis and answers to key questions.}
\emph{Exploration-exploitation in offline RL:} We do not explore; all learning occurs from logged data. Safety is enforced via the training distribution (safe-set restriction) and can be optionally enforced at deployment with a reject option when $\hat p_h(s)\ge \tau_{G(s)}$. \emph{Different action spaces across groups:} The primitive action set is shared. Training on group-safe sets changes class balance but not the catalog of actions. In operations, a guardrail can disallow actions in unsafe states; this affects \emph{when} actions are taken, not \emph{which} actions exist. \emph{Sensitivity to $\alpha,\epsilon$:} Ablations indicate monotone coverage-harm trade-offs; smaller $\alpha$ (coverage mode) and $\epsilon$ (harm mode) tighten safety at the cost of coverage. \emph{Why logistic regression for risk?} It yields calibrated, auditable probabilities and fast retraining; more expressive models are compatible but increase governance burden.

\paragraph{On identical or near-identical values across methods.} Our value proxies are deliberately conservative and stable: (i) rewards are sparse (mostly zero with rare negative harms), (ii) the simple FQE uses a small feature map to avoid environment/version brittleness, and (iii) policies are evaluated on the \emph{same} episodes. Under these conditions, different trained policies that mostly align with behavior on support will exhibit very similar $\hat V_0$. This is not a bug: DR estimates with 95\% CIs corroborate negligible differences between FG-FARL and Fair-BC (overlapping intervals; Table~\ref{tab:dr-by-group}), while fairness diagnostics (coverage/harm by group) reveal the intended differences. For readers desiring more separation, our code supports richer feature maps (parsed state features, polynomial expansions) and neural OPE (CQL+FQE), which we report as sensitivity rather than primary metrics due to instability across environments.

\section{Ethical Considerations}
Given the sensitive healthcare context, we emphasize: (i) FG-FARL allocates opportunity (coverage) or harm caps transparently across groups, but cannot repair upstream data bias; (ii) subgroup definitions must be governed and audited; (iii) deployment should include human oversight and continuous monitoring; (iv) fairness constraints may reduce aggregate efficiency; stakeholders must weigh trade-offs. This research used de-identified data and was reviewed by an independent IRB; it was approved with a waiver of consent by the Copernicus WIRB (WCG IRB).

\section{Policy Implications}
Health plans can operationalize FG-FARL to balance safety and equity in outreach and triage. Group thresholds offer transparent levers for governance, with auditable logs of coverage and harm trade-offs. The method is lightweight enough for routine recalibration and drill-down reporting.

\appendix
\section{Formal Statements and Proofs}
\paragraph{Setup and notation.} Fix a protected attribute $G:\mathcal{S}\to\mathcal{G}$. Let $Z=\hat p_h(S)$ denote the risk score on state $S$ drawn from the calibration distribution conditional on $G(S)=g$. Let $F_g$ be the CDF of $Z\mid G=g$ and $\widehat F_g$ its empirical CDF from $n_g$ exchangeable calibration samples. For $q\in(0,1)$, write $F_g^{-1}(q)=\inf\{t: F_g(t)\ge q\}$ and likewise for $\widehat F_g^{-1}(q)$.

\begin{theorem}[Groupwise coverage, coverage mode]\label{thm:coverage}
Suppose within-group exchangeability holds and the score $Z=\hat p_h(S)$ is almost surely monotone in the adverse class. Let $\tau_g=\widehat F_g^{-1}(1-\alpha)$ be the $(1-\alpha)$ empirical quantile computed on the calibration slice for group $g$. Then for any $\delta\in(0,1)$, with probability at least $1-\delta$ over the calibration sample,
\[
\Pr\big(Z\le \tau_g\mid G=g\big) \ge 1-\alpha - \varepsilon_{n_g}(\delta),\quad \text{where}\quad \varepsilon_{n_g}(\delta)=\sqrt{\tfrac{1}{2n_g}\ln\tfrac{2}{\delta}}.\footnote{We adopt the right-continuous CDF convention $F(\tau)=\Pr(Z\le \tau)$. For continuous or strictly monotone scores, $\Pr(Z<\tau)=\Pr(Z\le \tau)$.}
\]
Consequently, the safe set $\{s: \hat p_h(s)<\tau_g\}$ attains population coverage at least $1-\alpha-\varepsilon_{n_g}(\delta)$ within group $g$ with probability $\ge 1-\delta$. Finite-sample uncertainty on the \emph{empirical} coverage can be reported using Wilson intervals \citep{wilson1927probable}.
\end{theorem}

\begin{proof}
By the Dvoretzky–Kiefer–Wolfowitz inequality, $\sup_t\,|\widehat F_g(t)-F_g(t)|\le \varepsilon_{n_g}(\delta)$ with probability at least $1-\delta$ \citep{d1978convergence}. Let $\tau_g=\widehat F_g^{-1}(1-\alpha)$. Then $\widehat F_g(\tau_g)\ge 1-\alpha$ by definition of the empirical quantile. Hence, on the DKW event,
\[
F_g(\tau_g) \ge \widehat F_g(\tau_g) - \varepsilon_{n_g}(\delta) \ge 1-\alpha-\varepsilon_{n_g}(\delta).
\]
But $F_g(\tau_g)=\Pr(Z\le \tau_g\mid G=g)$ under the right-continuous convention, which yields the claim. The monotonicity-in-class assumption ensures that larger $Z$ monotonically corresponds to higher adverse likelihood, aligning the conformal acceptance set $\{Z<\tau_g\}$ with lower-risk states and avoiding set inversions.
\end{proof}

\begin{proposition}[Equal-group-weighted ERM equivalence]\label{prop:fairbc-eq}
Let $\mathcal{I}_g$ be the index set of safe samples in group $g$ and $n_g=|\mathcal{I}_g|$. Define the group-averaged empirical loss $\bar L_g(\theta)=\frac{1}{n_g}\sum_{i\in \mathcal{I}_g} \ell(\theta; s_i,a_i)$. Then the objective that gives equal weight to each group,
\[
\min_\theta \; \frac{1}{|\mathcal{G}|} \sum_{g\in\mathcal{G}} \bar L_g(\theta),
\]
is exactly empirical risk minimization with per-sample weights $w_i \propto 1/n_{g(i)}$ (inverse sample count within the group of $i$). Consequently, setting $w(s) \propto 1/\hat c_{G(s)}$ with $\hat c_g=n_g/\sum_h n_h$ implements the same equal-group-weighting principle used in Fair-BC.
\end{proposition}

\begin{proof}
Expand the objective:
\[
\frac{1}{|\mathcal{G}|} \sum_{g} \bar L_g(\theta)
= \frac{1}{|\mathcal{G}|} \sum_{g} \frac{1}{n_g} \sum_{i\in \mathcal{I}_g} \ell(\theta; s_i,a_i)
= \sum_{i} \Big( \frac{1}{|\mathcal{G}|\, n_{g(i)}} \Big) \ell(\theta; s_i,a_i).
\]
Thus the equal-group objective is equivalent to ERM with weights $w_i = \frac{1}{|\mathcal{G}|\, n_{g(i)}} \propto 1/n_{g(i)}$. Writing $n_{g(i)} = \hat c_{g(i)} \sum_h n_h$ gives $w_i \propto 1/\hat c_{g(i)}$, as implemented.
\end{proof}

\begin{theorem}[Existence of harm-capped threshold, harm mode]\label{thm:harm-existence}
Fix $g\in\mathcal{G}$. Define $h_g(\tau)=\mathbb{E}[\,\mathbb{1}\{r<0\}\mid \hat p_h(S)<\tau,\, G(S)=g\,]$ for $\tau\in[0,1]$, and let $\bar h$ be the global harm among safe states under the global threshold $\tau_{\mathrm{glob}}=F^{-1}(1-\alpha)$ computed on the calibration slice. Under (A1)--(A2),
Then for any $\epsilon\ge 0$, the feasible set $\mathcal{T}_g=\{\tau\in[0,1]: h_g(\tau)\le \bar h+\epsilon\}$ is a (possibly degenerate) interval $[0,\tau_g^*]$ with $\tau_g^*=\sup\mathcal{T}_g\in\mathcal{T}_g$. In particular, selecting $\tau_g^*$ maximizes coverage $\Pr(\hat p_h(S)<\tau\mid G=g)$ subject to the harm cap.
\end{theorem}

\begin{proof}
By (A2), the sublevel sets $\{\tau: h_g(\tau)\le c\}$ are intervals of the form $[0,\tau_c]$ for any $c\in[0,1]$; right-continuity ensures closedness. Taking $c=\bar h+\epsilon$ yields $\mathcal{T}_g=[0,\tau_g^*]$ with $\tau_g^*=\sup\mathcal{T}_g$ attained. Coverage within group $g$ is $\Pr(\hat p_h(S)<\tau\mid G=g)=F_g(\tau)$, which is non-decreasing in $\tau$ by (A1). Therefore, any feasible $\tau$ is dominated by $\tau_g^*$ in coverage, establishing optimality.
\end{proof}

\begin{theorem}[Consistency of empirical harm-capped selection]\label{thm:harm-consistency}
Let $\widehat h_g^{(n)}(\tau)$ be the empirical conditional harm rate computed from an exchangeable calibration sample of size $n_g$, and define the empirical selector $\widehat\tau_g=\sup\{\tau:\widehat h_g^{(n)}(\tau)\le \bar h+\epsilon\}$. Suppose (A1)–(A2) hold, $\widehat h_g^{(n)}$ converges to $h_g$ uniformly in $\tau$ in probability, and additionally:
\begin{itemize}
  \item[(A3)] $h_g$ is continuous at $\tau_g^*$ (no boundary jump).
\end{itemize}
Then $\widehat\tau_g\xrightarrow{p} \tau_g^*$, where $\tau_g^*$ is given by Theorem~\ref{thm:harm-existence}.
\end{theorem}

\begin{proof}
Fix $\eta>0$. By uniform convergence, there exists $N$ such that for $n_g\ge N$, with probability $\ge 1-\delta(\eta)$ we have $\sup_\tau |\widehat h_g^{(n)}(\tau)-h_g(\tau)|<\eta$. On this event,
\[
\{\tau: \widehat h_g^{(n)}(\tau) \le \bar h+\epsilon\} \subseteq \{\tau: h_g(\tau) \le \bar h+\epsilon+\eta\}
\quad \text{and} \quad
\{\tau: h_g(\tau) \le \bar h+\epsilon-\eta\} \subseteq \{\tau: \widehat h_g^{(n)}(\tau) \le \bar h+\epsilon\}.
\]
Define $\tau^+(\eta)=\inf\{\delta\ge 0: h_g(\tau_g^*+\delta) \ge \bar h+\epsilon+\eta\}$ and $\tau^-(\eta)=\inf\{\delta\ge 0: h_g(\tau_g^* - \delta) \le \bar h+\epsilon-\eta\}$. By (A2)–(A3), $\tau^+(\eta),\tau^-(\eta) \to 0$ as $\eta\to 0$. The containments above imply
\[
\tau_g^* - \tau^-(\eta) \le \widehat \tau_g \le \tau_g^* + \tau^+(\eta),
\]
and letting $\eta\downarrow 0$ yields $\widehat \tau_g \xrightarrow{p} \tau_g^*$. This is a standard argmax consistency argument under uniform convergence; see, e.g., \citep{vandervaart1998asymptotic}.
\end{proof}

\section{Additional Subgroup Ablations}
We include additional subgroup audits beyond age, race, and sex. For each attribute, we show safe-set coverage (coverage mode), harm within the safe set, and subgroup values.

% (subhead removed)
\begin{figure}[h]
  \centering
  \includegraphics[width=.32\textwidth]{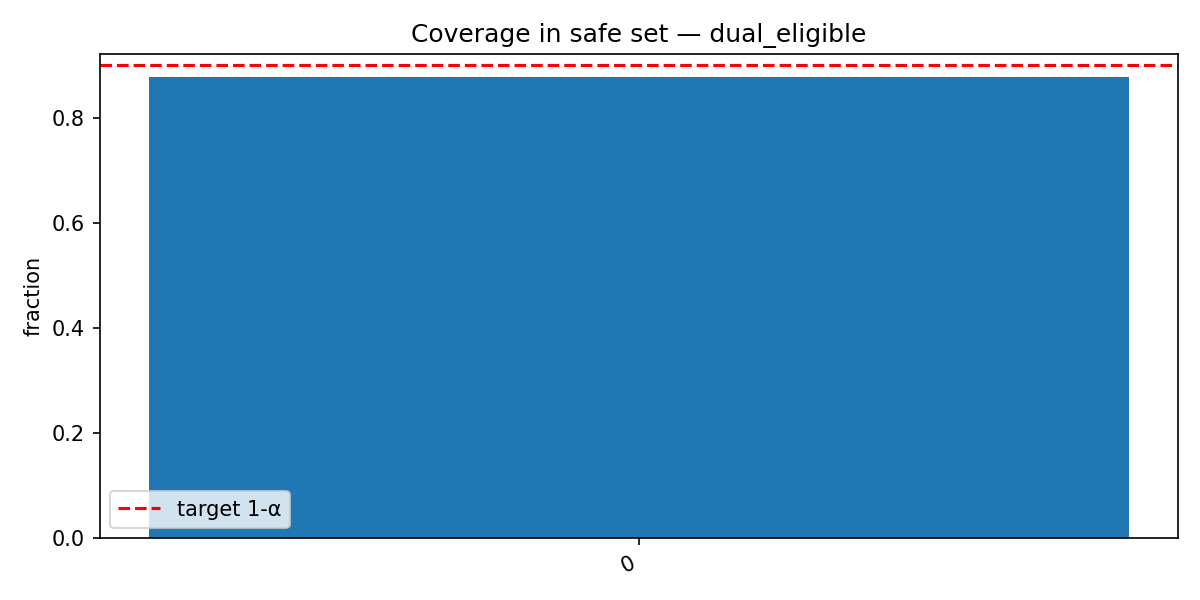}\hfill
  \includegraphics[width=.32\textwidth]{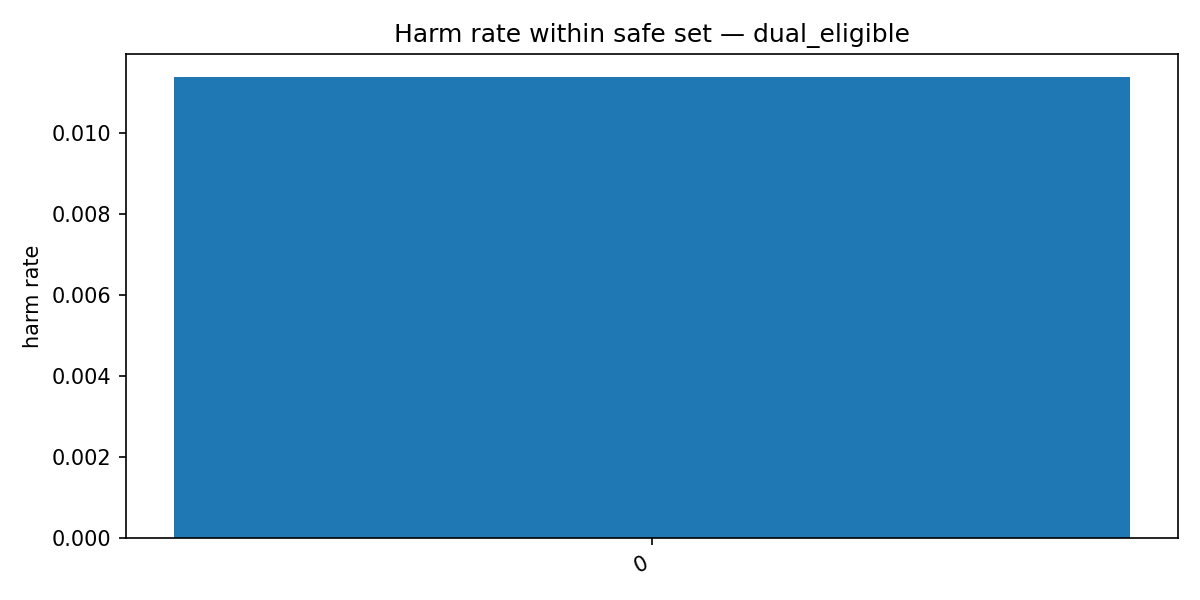}\hfill
  \includegraphics[width=.32\textwidth]{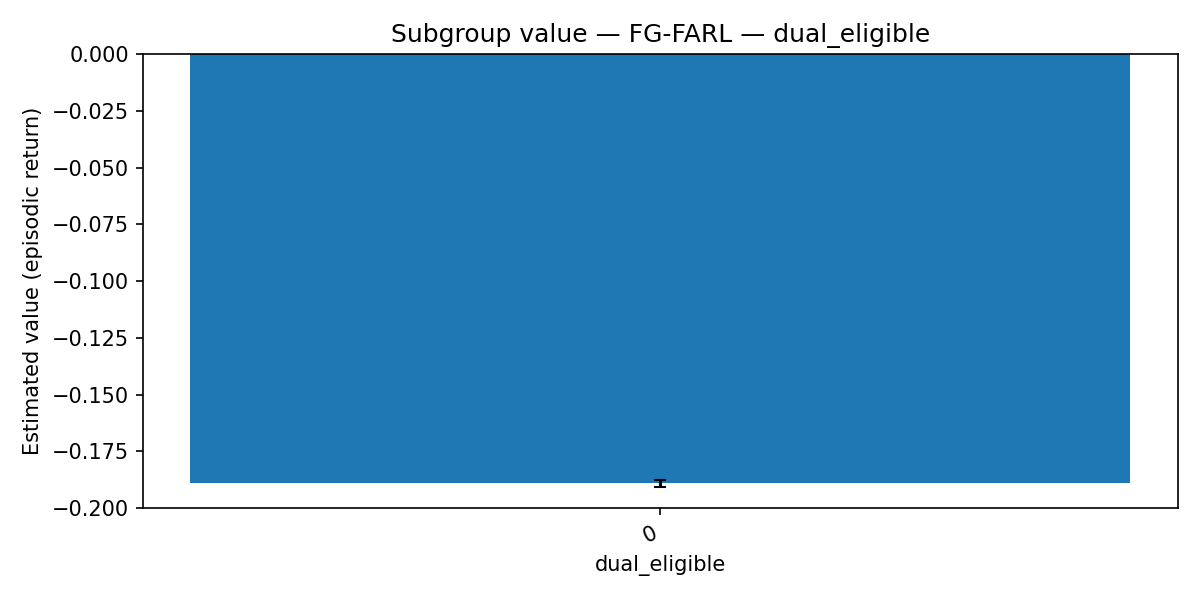}
  \caption{Dual eligibility (coverage mode).}
\end{figure}

\section{Policy Interpretability}
We report the most influential standardized coefficients from the multinomial logistic policy trained on the FG-FARL safe union (coverage mode, primary group=age). Coefficients are reported in standardized units (after imputation and scaling) and indicate the direction and magnitude of association between features and action logits.

\begin{table}[h]
  \centering
  \caption{Top coefficients per action in the multinomial logit policy (standardized units).}
  \label{tab:policy-top-features}
  \begin{tabular}{r l r}
\toprule
Action & Feature & Coefficient (std. units)\\
\midrule
0 & Age & 0.002\\
0 & Previous reward & 0.000\\
0 & Time index (t) & -0.004\\
1 & Age & 0.002\\
1 & Previous reward & 0.000\\
1 & Time index (t) & -0.001\\
2 & Time index (t) & 0.002\\
2 & Previous reward & 0.000\\
2 & Age & -0.003\\
3 & Time index (t) & 0.001\\
3 & Previous reward & 0.000\\
3 & Age & -0.000\\
4 & Age & 0.001\\
4 & Previous reward & 0.000\\
4 & Time index (t) & -0.001\\
5 & Time index (t) & 0.001\\
5 & Age & 0.000\\
5 & Previous reward & 0.000\\
6 & Time index (t) & 0.001\\
6 & Previous reward & 0.000\\
6 & Age & -0.002\\
7 & Age & 0.000\\
7 & Previous reward & 0.000\\
7 & Time index (t) & -0.001\\
8 & Time index (t) & 0.001\\
8 & Age & 0.000\\
8 & Previous reward & 0.000\\
\bottomrule
\end{tabular}
\end{table}

% (subhead removed)
\begin{figure}[h]
  \centering
  \includegraphics[width=.32\textwidth]{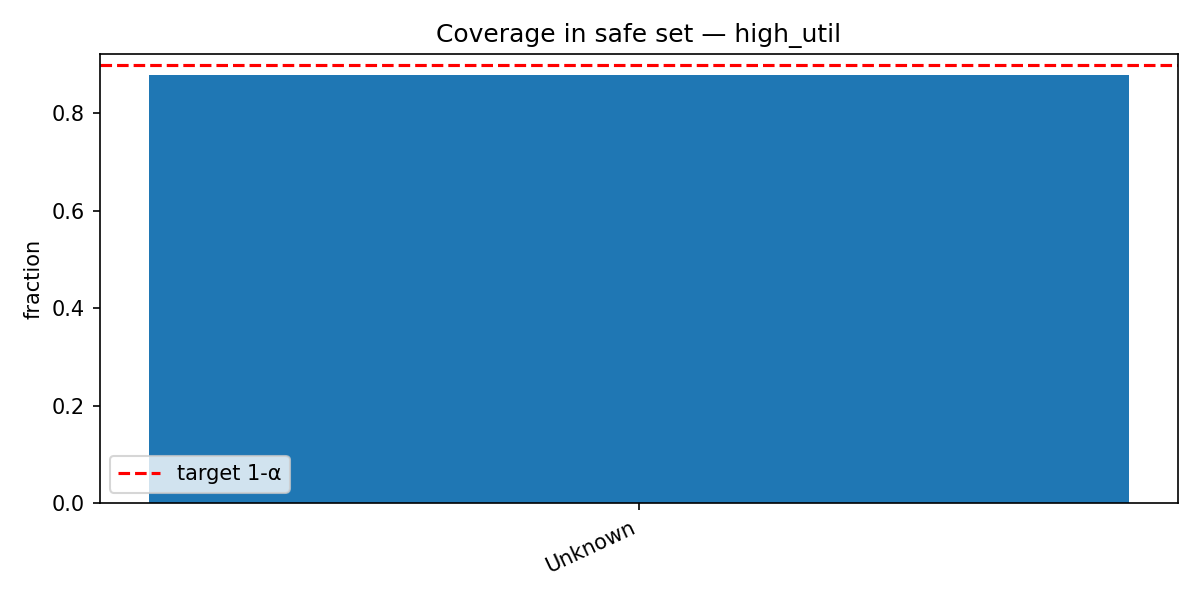}\hfill
  \includegraphics[width=.32\textwidth]{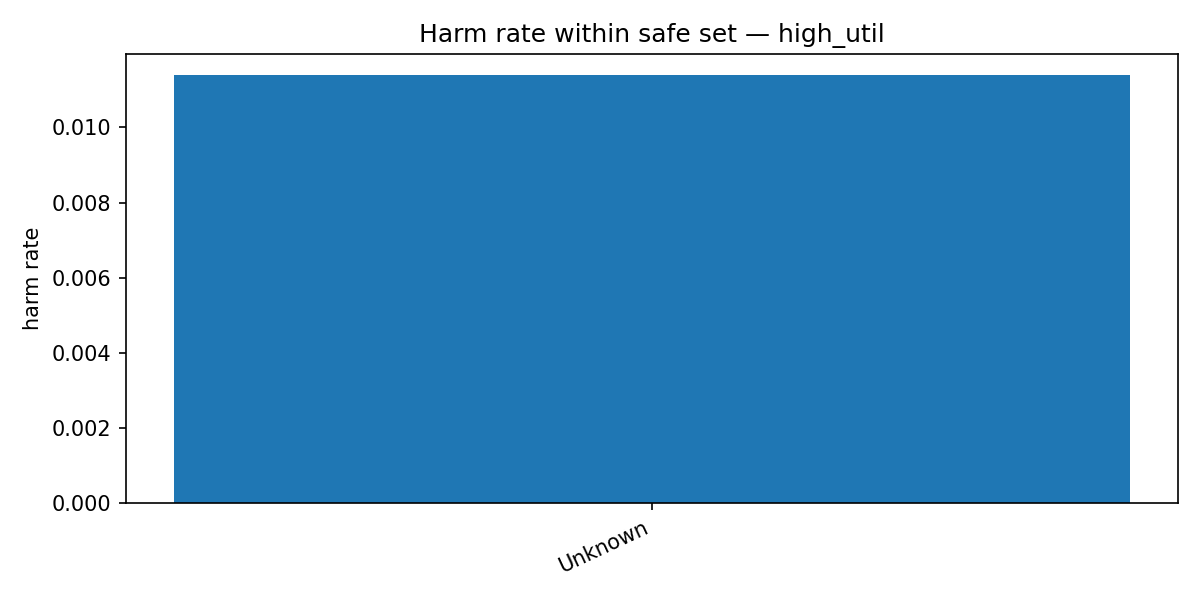}\hfill
  \includegraphics[width=.32\textwidth]{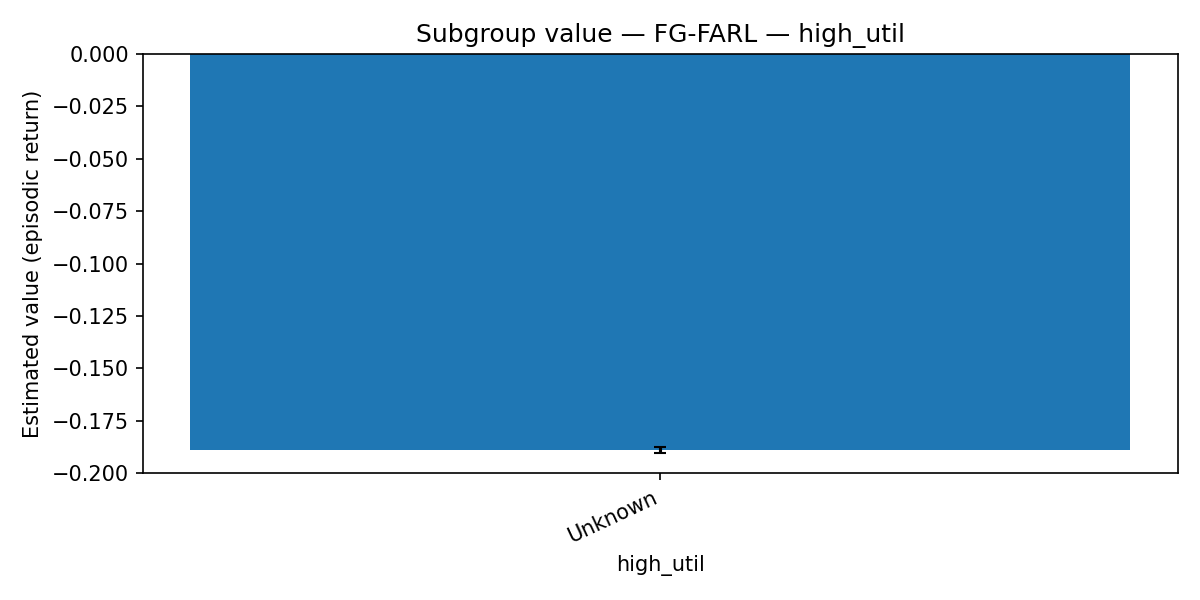}
  \caption{High utilization (coverage mode).}
\end{figure}

% (subhead removed)
\begin{figure}[h]
  \centering
  \includegraphics[width=.32\textwidth]{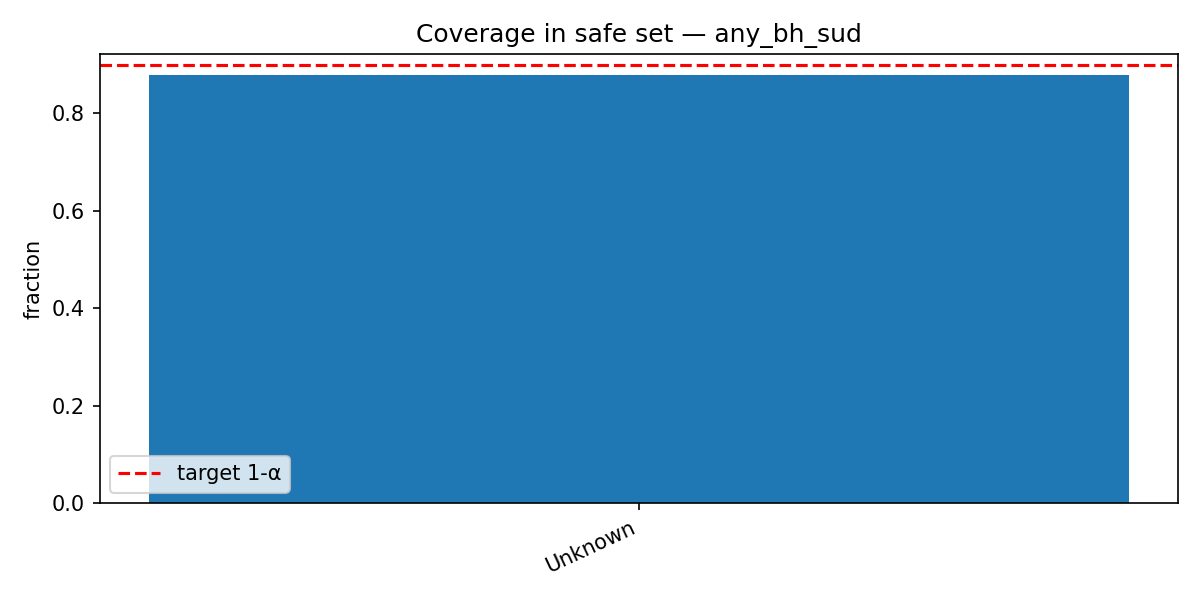}\hfill
  \includegraphics[width=.32\textwidth]{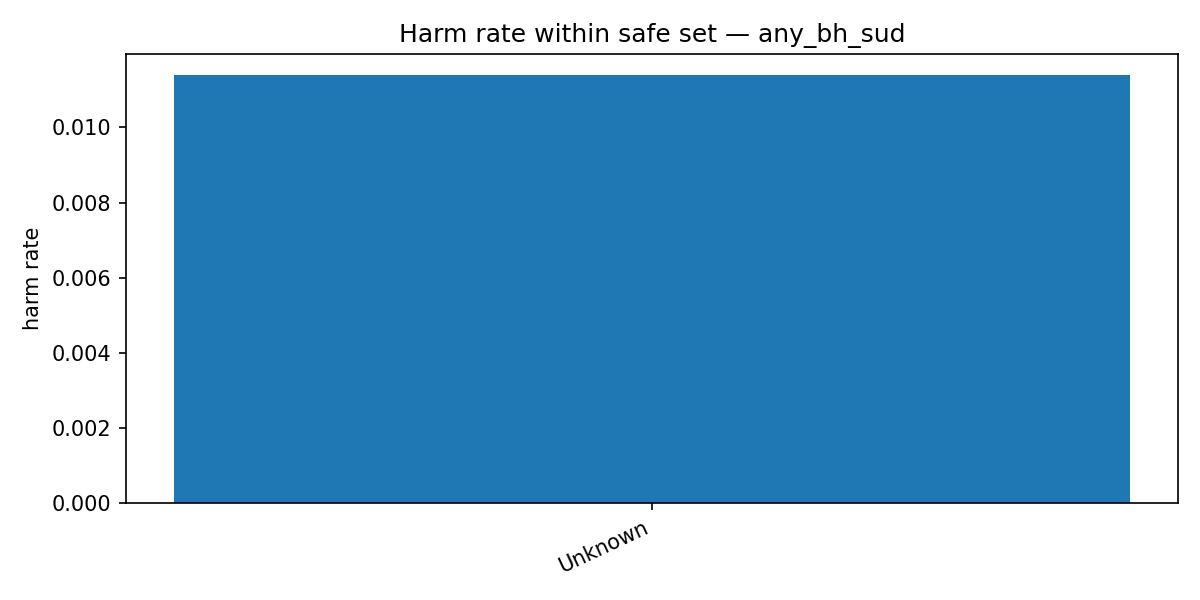}\hfill
  \includegraphics[width=.32\textwidth]{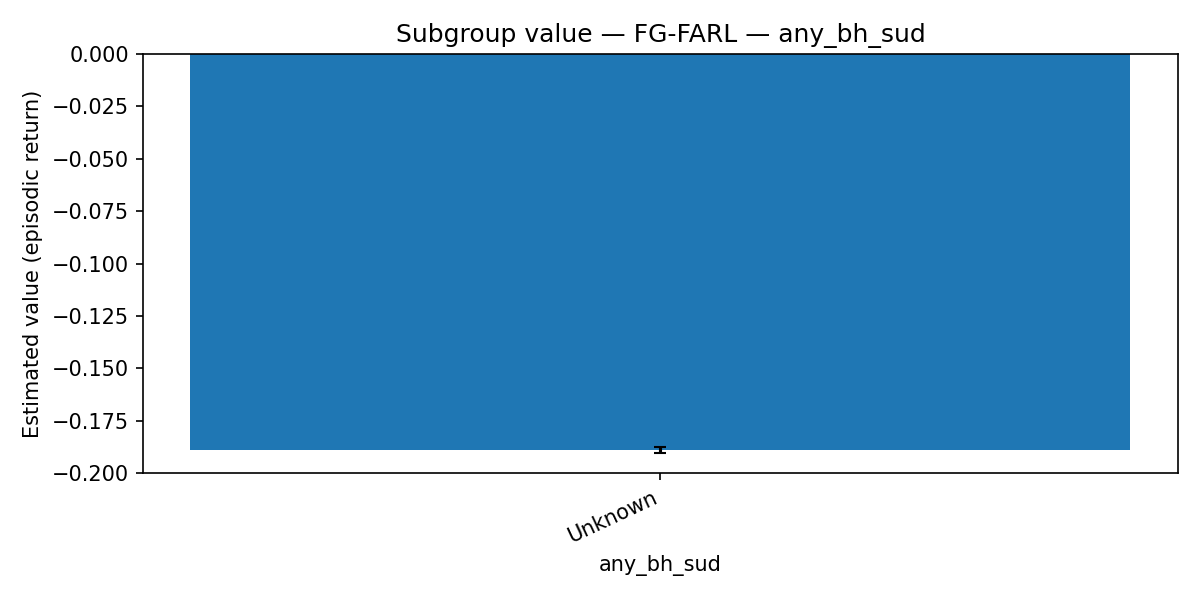}
  \caption{Any BH/SUD (coverage mode).}
\end{figure}

% (subhead removed)
\begin{figure}[h]
  \centering
  \includegraphics[width=.32\textwidth]{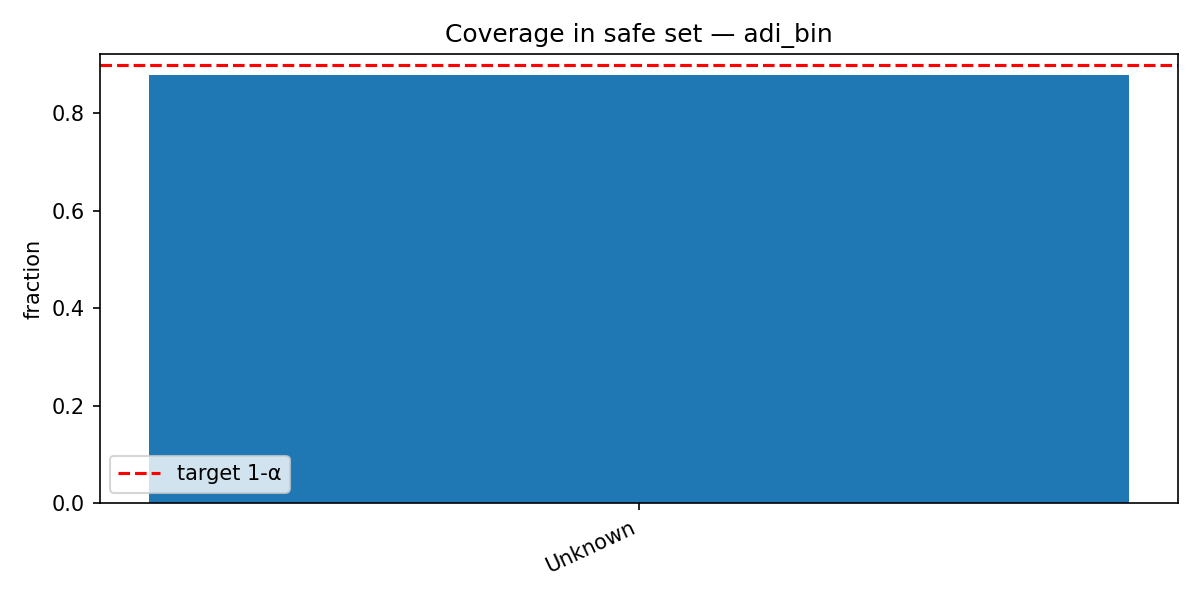}\hfill
  \includegraphics[width=.32\textwidth]{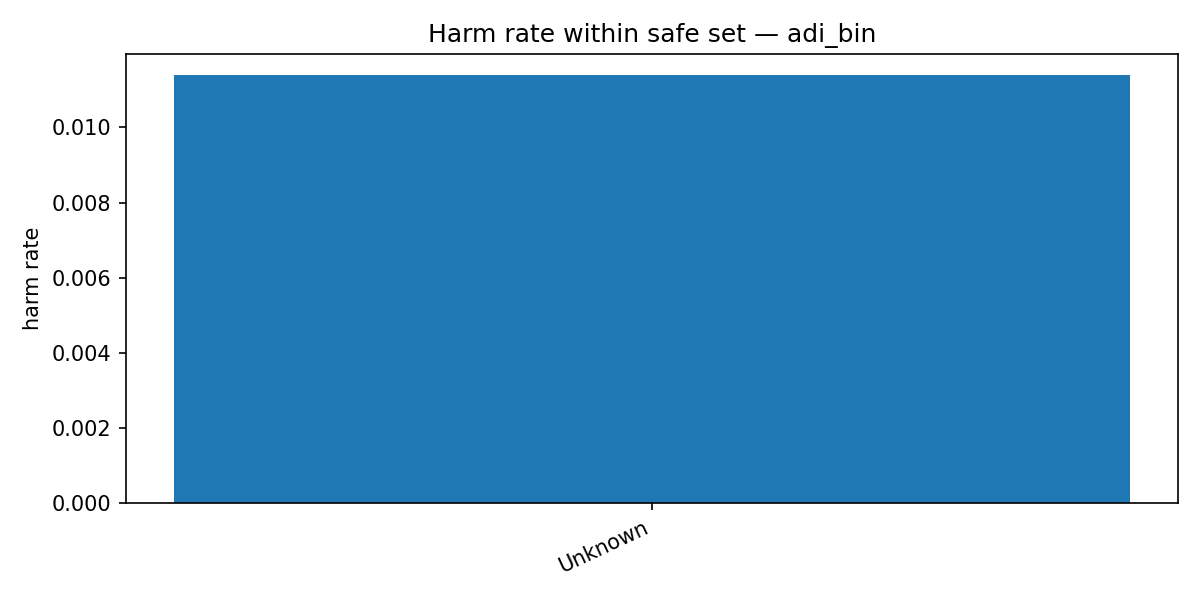}\hfill
  \includegraphics[width=.32\textwidth]{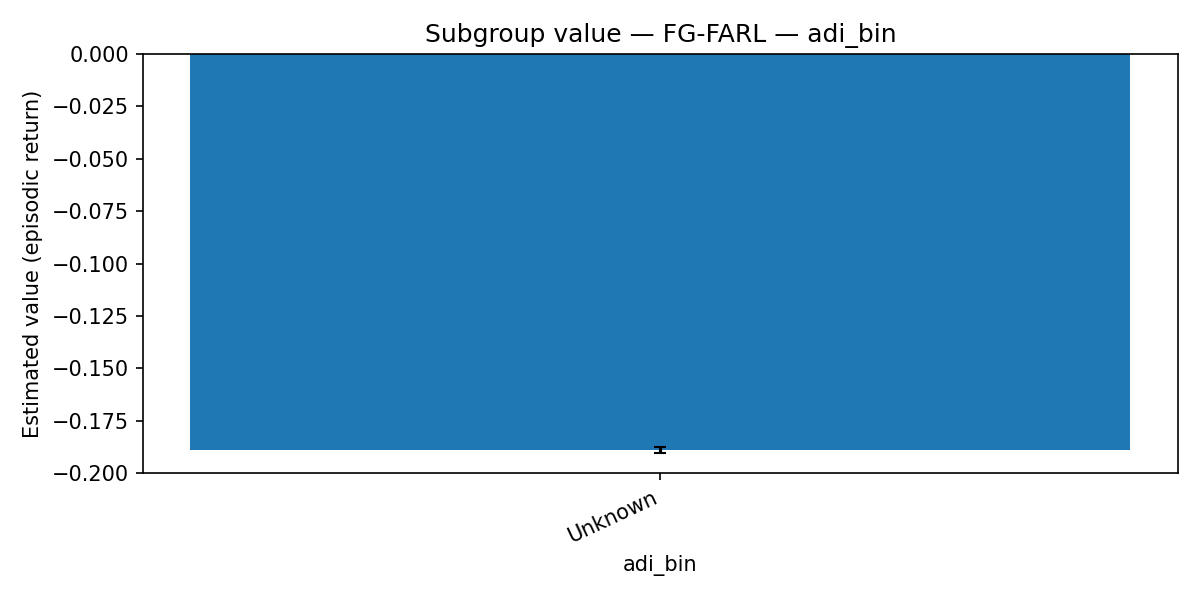}
  \caption{ADI bin (coverage mode).}
\end{figure}

\section{Reproducibility}
We fix seeds for data splits and evaluation subsets, and we report all hyperparameters and environment knobs (\texttt{ALPHA}, \texttt{EPSILON}, \texttt{MIN\_GROUP\_N}, \texttt{REWARD\_NORM}). We report value estimates with bootstrap confidence intervals and provide subgroup auditing diagnostics.\newline
\textbf{Code availability.} An end-to-end implementation with scripts to reproduce tables and figures is available at \url{https://github.com/sanjaybasu/fg_farl/tree/main}.

\bibliographystyle{unsrtnat}
\bibliography{refs}

\end{document}